\documentclass[twoside,11pt]{article}

\usepackage{jmlr2e} 

\usepackage{natbib}

\usepackage{todonotes}

\usepackage{caption}
\usepackage{subcaption}
\usepackage{url}
\usepackage{amssymb}
\usepackage{graphicx}
\usepackage{amsmath}
\usepackage{color}
\usepackage{listings}
\usepackage[framed, numbered]{matlab-prettifier}
\usepackage[T1]{fontenc}
\usepackage{float,graphicx,epsfig,amsmath,graphpap,matlab-prettifier, bm}
\usepackage[vmargin=3cm, hmargin=3cm]{geometry}
\usepackage[noend]{algpseudocode}
\usepackage{algorithm}

\newcommand{\one}{\mathbf{1}}
\newcommand{\Pro}{\mathbb P}
\newcommand{\E}{\mathbb E}
\newcommand{\iy}{\infty}

\newcommand{\RR}{\mathbb{R}}
\newcommand{\bx}{\bm{x}}
\newcommand{\bX}{\bm{X}}

\newcommand{\bs}{\bm{s}}
\newcommand{\bv}{\bm{v}}
\newcommand{\bu}{\bm{u}}
\begin{document}

\firstpageno{1}

\title{Robust Neural Network Classification via Double Regularisation}

\author{\name Olof Zetterqvist \email olofze@chalmers.se \\
       \addr Mathematical Sciences\\
       Chalmers University of Technology and University of Gothenburg\\
       Gothenburg, Sweden \\
       Research sponsored by Wallenberg AI, Autonomous Systems and Software Program (WASP)
       \AND
       \name Rebecka J\"ornsten \email jornsten@chalmers.se \\
       \addr Mathematical Sciences\\
       Chalmers University of Technology and University of Gothenburg\\
       Gothenburg, Sweden \\
       Research supported by SSF (Swedish Foundation for Strategic Research) and the Swedish Research Council.
       \AND
       \name Johan Jonasson \email jonasson@chalmers.se \\
       \addr Mathematical Sciences\\
       Chalmers University of Technology and University of Gothenburg\\
       Gothenburg, Sweden \\ Research supported by Wallenberg AI, Autonomous Systems and Software Program (WASP) and Centiro Solutions AB.}
       
\editor{XXX}


\maketitle

\begin{abstract}
The presence of mislabelled observations in data is a notoriously challenging problem in statistics and machine learning, associated with poor generalisation properties for both traditional classifiers and, perhaps even more so, flexible classifiers like neural networks.
Here we propose a novel double regularisation of the neural network training loss that combines a penalty on the complexity of the classification model and an optimal reweighting of training observations. 
The combined penalties result in improved generalisation properties and strong robustness against overfitting in different settings of mislabelled training data and also against variation in initial parameter values when training.
We provide a theoretical justification, by proving that for logistic regression with multivariate Gaussian covariates,
our proposed method can find the correct parameters exactly, i.e.\ estimate the parameters to exactly the same value as if there were no mislabelling.
We demonstrate the double regularisation model, here denoted by DRFit, for neural net classification of (i) MNIST and (ii) CIFAR-10, in both cases with simulated mislabelling. 
We also illustrate that DRFit identifies mislabelled data points with very good precision.
This provides strong support for DRFit as a practical of-the-shelf classifier, since, without any sacrifice in performance, we get a classifier that simultaneously reduces overfitting against mislabelling and gives an accurate measure of the trustworthiness of the labels.
\end{abstract}

\begin{keywords}
  Deep Learning, Robust Statistics, Regularisation, Mislabelling, Mislabelled data, Contaminated labels, DRFit, weighted loss.
\end{keywords}

\newpage

\section{Introduction}

In 
supervised, semi-supervised or active learning with machine learning algorithms in general, and deep neural nets in particular, the default models assume that the given annotations of training data are correct. However, since well-annotated data sets can be expensive and time consuming to collect, a fair amount of recent research has focused on using larger but noisy sets of training data. Such data sets are much cheaper to collect, e.g.\ via crowd sourcing. The working assumption is that collecting larger amounts of data that are labelled with a decent accuracy can compensate for the noise (inaccurate labels).

Clearly, 
it is desirable that parameter estimation in neural networks is robust 
with respect to changing a small fraction of the labels 
of the training data. Results in the literature are conflicting, with findings 
depending on the noise distribution, the type of neural net, the amount of correctly labelled data and level of contamination. 
While some results indicate that classification performance on correct labels can be noise stable (e.g.\ \citep{rolnick2017deep}), there are many studies that have shown that neural networks are
highly sensitive to label contamination \citep{pmlr-v70-koh17a,zhang2016understanding,arpit2017closer}.

There is 
a variety of methods for dealing with noisy labels \citep{NoiseSurvey,song2020learning}. One approach is to try to identify training examples that are likely to be mislabelled and correct them manually \citep{pmlr-v70-koh17a,pleiss2020identifying}.
This is computationally costly and suffers from a chicken-and-egg problem; examples that stand out as likely to be incorrectly labelled do so on the basis of a model that is in itself trained on noisy data.

Some methods focus on using data augmentation in order to  reduce the impact of noisy labels \citep{zhang2017mixup}, while others estimate a confusion matrix that compensates for the noise in the labels \citep{goldberger2016training,hendrycks2018using}. 

Another approach is to incorporate the noise directly into the loss function  of the training model \citep{43273,NIPS2013_5073, sukhbaatar2015training, LiuReweighting,tanaka2018joint,arazo2019unsupervised}, either via a surrogate loss function that augments the classification loss with noise rate parameters or with a data reconstruction term.
For all the methods mentioned, overfitting can still be a problem,
which is then usually addressed via early stopping.
The intuition behind early stopping is that the optimisation algorithm at its early stages finds the large scale classification boundaries and only after that starts to fine-tune to individual training examples and thereby overfit. This intuition is supported by recent theoretical results (e.g.\ 
\citep{pmlr-v108-li20j}). Early stopping results in the final network parameters to be a mix of starting values and optimal values. Several questions arise; how should the starting values be chosen?, when should we stop?. Also, from a statistical point of view, it is very unsatisfactory 
to define a model (loss) that we 
in some sense do not ultimately use.

Another alternative to overcome overfitting is
through some form of regularisation on the network parameters, e.g.\ lasso or ridge. However, we will here show that this does not fully alleviate overfitting in the presence of training data mislabelling.

In this paper we suggest a double regularisation (DR) technique where,
in addition to using lasso or ridge, each term in the loss function is multiplied with an observation weight. The observation weights will be penalised for deviating from all being equal. The weights can thus be seen as a new set of model parameters, as many as the number of observations, apparently adding to the complexity of the model. However, with the proper choice of regularisation penalty for the observation weights, when optimising this new loss function, it is easy to solve for the observation weights in terms of the original parameters and thereby obtain an explicit new loss function, where in fact no new parameters have been added and hence no extra complexity has been imposed. 
The benefits of the new model
 include; (a) a mathematical formulation that explains the weighting methodology as a regularisation mechanism and (b) circumventing the need for early stopping since the loss function is explicitly optimised.
 
We claim that intuitively this model does not overfit, since early on during training, the weights will adjust 
so that observations that stand out too much will be largely ignored. Hence the observation weights work in themselves
as early stopping but without the need to actually stop. 
Experimental results very strongly suggest this intuition is correct. 
In addition, in Section \ref{section_theory} we provide a theoretical justification, showing that for logistic regression, DRFit performs strictly better than a standard regularised neural network. Indeed, for multivariate Gaussian covariates with equal covariance structure in the two classes, DRFit finds the correct parameters {\em exactly}.

\section{Methods}

Consider a classification problem in the space $X \times Y$ into $K$ different classes, where $X$ is the input space and $Y$ the label space and let $z_i = (x_i,y_i) \in X \times Y$, $i=1,\ldots,n$ be the observations. For inference, we use a model $y=f(x;\theta)$, where $\theta \in \Theta$ are the model parameters and $x \in X$. In order to find the optimal $\theta$, a loss function $C(\theta;z) = \sum_i L(\theta;z_i)$ is minimised. 

Now consider uniform label noise such that each training point is contaminated with a probability that may depend on class but is otherwise independent of the features. In a setting with the presence of mislabelled data, a model with even a few parameters can easily overfit in the sense that it eventually learns a false and overly complex structure, trying to separate between instances in the training data that have different labels but are in truth of the same class.
As a result, the model generalises poorly to new data. 
To alleviate this problem, we propose a regularisation technique that gives each training point an observation weight that reflects how much we ``believe'' in the annotation of that point, combined with a penalty for observation weights for deviating too much from one (since without such a penalty, the model would simply put all weight on one example from each class). The idea is that our model will learn observation weights that are close to $0$ for mislabelled data while leaving other weights well away from $0$.
When the model is also over-specified, we combine this observation weighting with some standard regularisation on the model parameters, such as a ridge or lasso regularisation.  

The minimisation object of our doubly regularised model, DRFit, can in its general form thus be written as

\begin{align}
    (\hat{\theta}, \hat{\omega}) = \text{argmin}_{\theta,\omega} \sum_i \omega_i L(z_i,\theta) + \alpha g(\omega) + \lambda \tilde g(\theta),
    \label{ref:min}
\end{align}

where $\alpha$ and $\lambda$ are hyper-parameters controlling the amount of regularisation, and $g$ and $\tilde g$ are regularisation functions which are typically convex functions and where the last term may be dropped if the model is not over-parametrised. For the rest of this paper, we will consider $\tilde g = \frac{1}{2}||\cdot ||_2^2$, i.e.\ a ridge regularisation on the original model parameters, or $\tilde g = 0$ in case of no overspecification.

There are some recent studies that utilise observation weights in the context of regularised or robust regression \citep{luo2016robust, gao2016penalized, gao2017penalize} as well as some methods on neural networks \citep{ren2018learning} where the observation weights are calculated based on a comparison between the gradients of training and validation data. With the extra forward and backward passes
this algorithm requires approximately three times as many computations as a regular training algorithm.

The two hyperparameters in (\ref{ref:min}) control the trade-off between observation weighting and regularisation of the network parameters. A larger $\lambda$ reduces the complexity of the neural network and a larger $\alpha$ restricts the flexibility of observation reweighting.
While it may appear as if we have introduced a large number of extra parameters with observation weights, as mentioned above, 
with a certain natural choice of regularisation function $g$, we will be able to explicitly solve for the observation weights in terms of the model parameters, hereby getting a new modified minimisation objective explicitly expressed in terms of the loss function $L(\theta;z)$.  

\section{Choice of weight regularisation function} \label{Choice_of_weight_regularisation_function}

In this paper we will use an entropy regularisation on the observation weights: $g(\omega_i) = \omega_i \log(\omega_i) - \omega_i$ combined with the constraint that $\sum_{i:y_i \in C_k}\omega_i = \rho_k|C_k|$ for each class $k$. Here $C_k$ is the set of training examples that have been labelled $k$ and $\rho_k$ is a scaling factor that can be tuned to fit the label noise. The main purpose of the $\rho_k$'s is to make up for the difference in class proportions as they appear in the noisy training data and the true proportions. To estimate the true proportions, a natural assumption is that we have access to a clean set of validation data. If not, the natural choice is to set $\rho_k=1$ for all $k$.

As promised, we can now solve for the observation weights in terms of the original model parameters as follows. With the given $g$, the DRFit optimisation problem reads:

\begin{equation} \label{min_prob_analytic}
(\hat{\theta}, \hat{\omega}) = \text{argmin}_{\theta,\omega} \sum_i \left[ \omega_i L(z_i,\theta) + \alpha (\omega_i \log(\omega_i) - \omega_i) \right] + \lambda \frac{1}{2}||\theta||_2^2. 
\end{equation}

Solving for $\omega$ in terms of $\theta$ in (\ref{min_prob_analytic}) gives the following:

\begin{proposition} \label{prop:new_loss}
    Solving the minimisation problem (\ref{min_prob_analytic}), i.e.\ solving
    
    \[(\hat{\theta}, \hat{\omega^*}) = \textnormal{argmin}_{\theta,\omega} f(\theta,\omega),\]
    where
    \[f(\theta,\omega) = \sum_i \left[ \omega_i L(z_i,\theta) + \alpha (\omega_i \ln(\omega_i) - \omega_i) \right] + \lambda \frac{1}{2}||\theta||_2^2,\]
    
    with the constraints $\sum_{i \in C_k} \omega_i = \rho_k|C_k|$, $k=1,\ldots,K$ is equivalent to solving
    
    \begin{equation} \label{target}
       \hat{\theta} = \textnormal{argmin}_{\theta} h(\theta), 
    \end{equation}
    where
    \[h(\theta) = -\alpha \sum_k 
       \rho_k|C_k| \log\left(\sum_{i \in C_k} e^{-\frac{L(z_i,\theta)}{\alpha}}\right) + \frac{\lambda}{2}||\theta||_2^2, \]
    and then taking
    \[ \omega_i = \rho_k|C_k| \frac{e^{-L(z_i,\hat{\theta})/\alpha}}{\sum_{j \in C_k} e^{-L(z_j,\hat{\theta})/\alpha}}.\]
    
\end{proposition}

\begin{proof}
For convenience, we let $\ell_i = L(z_i,\theta)$ and $n_k = \rho_k|C_k|$ 
and 
\[
F(\theta,\omega,r) = \sum_i \omega_i \ell_i + \alpha (\omega_i \log(\omega_i) - \omega_i) + \lambda \frac{1}{2}||\theta||_2^2 + r_k\left(\sum_{i \in C_k} \omega_i - n_k\right),
\]
where $r_k$ is a Lagrange factor associated with class $k$.
Our goal is now to find stationary points of $F(\theta,\omega,r)$. 
Differentiate with respect to $\omega_i$ and set to $0$. This gives 
\begin{equation*}
    \ell_i + \alpha \log(\omega_i) + r_k = 0,
\end{equation*}
which has the unique solution for $\omega_i$ for each $i$
\[\omega_i = e^{-(\ell_i+r_k)/\alpha}\]
which thus has to be satisfied for all $i$ for any stationary point of $F$.

Inserting into the constraints $\sum_{i \in C_k} \omega_i = n_k$ and solving for $r_k$, we further get
\begin{equation} \label{ea}
    \omega_i = n_k \frac{e^{-\ell_i/\alpha}}{\sum_{j \in C_k} e^{-\ell_j/\alpha}}.    
\end{equation}
Another way of putting this is that the global minimum of $f(\theta,\omega)$ belongs to the subspace $D$ of pairs $(\theta,\omega)$ that satisfy (\ref{ea}) for all $i$. Conversely, since any $\theta$ corresponds to a unique $(\theta,\omega) \in D$ via (\ref{ea}), inserting (\ref{ea}) back into $f$, solving for $\theta$ to get the minimum $\hat{\theta}$ and then taking $\hat{\omega}_i$ by inserting $\hat{\theta}$ into (\ref{ea}), gives the global minimum of $f(\theta,\omega)$. It remains to show that inserting (\ref{ea}) into $f(\theta,\omega)$ gives $h(\theta)$. Doing the insertion and for convenience dropping the ridge term (which does not depend on $\omega$) gives

\begin{align*}
&
\sum_i \omega_i \ell_i + \alpha (\omega_i \log(\omega_i) - \omega_i) \\
&= \sum_k \sum_{i \in C_k} \omega_i \ell_i + \alpha (\omega_i \log(\omega_i) - \omega_i)  \\ 
&= \sum_{k} \sum_{i \in C_k} n_k \frac{e^{-\ell_i/\alpha}}{\sum_{j \in C_k} e^{-\ell_j/\alpha}} \ell_i +
\alpha \left(n_k \frac{e^{-\ell_j/\alpha}}{\sum_{j \in C_k} e^{-\ell_j/\alpha}} 
\, \log \left(n_k \frac{e^{-\ell_i/\alpha}}{\sum_{j \in C_k} e^{-\ell_j/\alpha}}\right) - n_k \frac{e^{-\ell_i/\alpha}}{\sum_{j \in C_k} e^{-\ell_j/\alpha}}\right) \\
&= -\alpha \sum_k n_k \log\left(\sum_{j \in C_k} e^{-\ell_j/\alpha}\right) + n_k \log(n_k) - n_k.
\end{align*}

Since $n_k$ is independent of $\theta$ and $\omega$, we are done.
\end{proof}

In summary, using $g(\omega) = \sum_i (\omega_i\log(\omega_i) - \omega_i)$ with the constraints $\sum_{i \in C_k}\omega_i=\rho_k|C_k|$, we simply end up with a new minimisation objective with no extra parameters  (besides
the hyperparameter $\alpha$ which still controls how much the observation weights are allowed to deviate from equal).

\section{Theoretical results} \label{section_theory}
\medskip

In what follows, we will make a theoretical justification of the weight penalty in the setting of binary classification with logistic regression. This will first be done for the very simplistic case of one-dimensional covariates without intercept term and then expanded to the multivariate setting. Conditions for the results will be gradually stronger, but in each case still natural. 
The model will be assumed to not be be overparametrised; the number of parameters $r$ will regarded as fixed and results will be for limits as the number of observation grows to $\iy$. Hence we discard the ridge penalty term from (\ref{target}).
The model may thus be written as
\[\mathbb{P}(y^*_i=1|\bx_i) = \frac{e^{\bs^T\bx_i}}{e^{\bs^T\bx_i}+e^{-\bs^T\bx_i}}\]
for an unknown vector of parameters $\bs=[(s_1,\ldots,s_p)^T$, where we write $\bx_i$ for the vector of covariates associated with observation $i$ and $y^*_i$ for the true label of the $i$th data point to distinguish it from the possible erroneous label $y_i$; we will consistently use notation with $*$'s to denote quantities corresponding to true labels. 

For standard logistic regression, $b\bs$ is estimated as
\begin{equation} \label{slogist}
    \hat{\bs}^*_n=\text{argmax} \, \mathcal{L}^*_n(\bs),
    \end{equation}
    where
    \begin{align*} n\mathcal{L}^*_n(\bs) &= n\mathcal{L}^*_n(\bs;\mathbf{x},\mathbf{y}) \\ 
    &= \bs^T\left(\sum_{i:y^*_i=1}\bx_i - \sum_{i:y^*_i=0}
\bx_i \right) -\sum_{i=1}^n \log(e^{\bs^T\bx_i}+e^{-\bs^T\bx_i}).
\end{align*}

Let us now focus on the case $r=1$ with no intercept term for a while, i.e.\ on the model
\[\mathbb{P}(y^*_i=1|x_i) = \frac{e^{sx_i}}{e^{sx_i}+e^{-sx_i}}\]
for an unknown one-dimensional slope $s$.

We claim that under various natural distributions of the covariates, one can always find an $\alpha$ such that (\ref{target}) used on noisy data will, in the limit as $n \rightarrow \iy$, produce an estimate of $s$ that exactly coincides with the one for clean data; this is arguably best possible.  

\medskip

We start by showing that under very mild assumptions, contamination of the labels will always cause standard logistic regression to underestimate $s$. This is formalised by a proposition below. Since the introduction of a fair amount of notation will be needed to state the proposition, we introduce it along with proving the proposition and state the result as a summary.

Let the true data be $(x_i,y^*_i)$, $i=1,\ldots,n$ and let $y_i$ be the $i$'th label in the contaminated data set. 
To model data generation, let $p^*_k:=\Pro(y^*_i=k)$, $k=0,1$, the probability that a data point has true label $k$, where we make the obvious assumption that $p_1^*,p_0^* > 0$.
For the remainder of this section, the following assumptions are also made.

\begin{itemize}
    \item[(i)] Covariates whose true label is $k$ are distributed according to density $f^*_k$, $k=0,1$.
    \item[(ii)] Writing $X^*_k$ for a random variable distributed according to $f^*_k$, it holds that $\E[X^*_1]>0$, $\E[X^*_0] < 0$, $\Pro(X^*_1<0)>0$ and $\Pro(X^*_0>0)>0$.
\end{itemize} 

These assumptions ensure that the two distributions are overlapping but not equal. They also entail that $0 \leq \hat{s}^* < \infty$. 

As $n \rightarrow \infty$, $\mathcal{L}^*_n(s)$ converges uniformly on bounded sets to
\begin{equation}
\label{slogistlimit} 
   \mathcal{L}^*(s):= s(p^*_1\E[X^*_1]-p^*_0\E[X^*_0]) - \E[\log(e^{sX}+e^{-sX})],
\end{equation}

where $X=X^*$ is chosen according to the overall distribution of the covariates, and hence $\hat{s}^*:=\lim_{n \rightarrow \iy}\hat{s}^*_n$ is the $\text{argmax}$ of (\ref{slogistlimit}). The corresponding generalisation of this to $r \geq 2$ also holds.

Now assume that true labels $k$ are erroneously classified as the other class independently across observations (given $k$) but with a probability potentially dependent on $x_i$. This gives us the contaminated labels $y_i$ and we want to control what happens when the $y^*_i$:s are replaced by $y_i$ in the above. 

Let $X_k$ be a random variable with distribution as $x_i$ given $y_i=k$. In addition to (i) and (ii), we add the following assumptions. Here $p_k$ denotes the probability that a given observation is labelled (potentially erroneously) as $k$ and $X_k$ is the distribution of $x_i$ given $y_i=k$.
\begin{itemize}
    \item[(iii)] $\E[X^*_1] \geq \E[X_1]$ and $\E[X^*_0] \leq \E[X_0]$.
    \item[(iv)] $p^*_1\E[X^*_1]-p^*_0\E[X^*_0] > p_1\E[X_1]-p_0\E[X_0]$.
\end{itemize}
Condition (iii) is very natural as the possible mislabels tend to mix up $X^*_0$ and $X^*_1$. Condition (iv) can easily be shown to follow from the other assumptions when the mislabelling probability of a data point is independent of $x_i$ given $y^*_i$ (i.e.\ the probability may depend on its true label, but in no other way on $x_i$).

Let
\[\hat{s} = \text{argmax} \, \mathcal{L}(s)\]
where
\[\mathcal{L}(s)=s(p_1\E[X_1]-p_0\E[X_0]) - \E[\log(e^{sX}+e^{-sX})],\]
i.e.\ the limit estimate of $s$ with noisy data.

\medskip

Now $\hat{s}$ is the solution of
\[\frac{\partial \mathcal{L}(s)}{\partial s} = p_1\E[X_1]-p_0\E[X_0] - \E[X\tanh{sX}] = 0\]
and $\hat{s}^*$ is the solution of
\[p^*_1\E[X^*_1]-p^*_0\E[X^*_0] - \E[X\tanh(sX)] = 0.\]
Since $p^*_1\E[X^*_1]-p^*_0\E[X^*_0] > p_1\E[X_1]-p_0\E[X_0]$ and
$(\partial/\partial s)X\tanh(sX) = 4X^2/(e^{sX}+e^{-sX})^2 > 0$, so that (by the Dominated Convergence Theorem) $\E[X\tanh{sX}]$ is increasing in $s$, we get $\hat{s} < \hat{s}^*$. We have proved:

\begin{proposition} \label{prop:underestimate}
    For $p=1$ and under (i)-(iv), standard logistic regression underestimates $s$, i.e.\
    \[\hat{s} < \hat{s}^*.\]
\end{proposition}

Next we claim that under the reigning assumptions, optimising (\ref{target}) without the ridge penalty term with $\alpha$ sufficiently small will, in the limit as $n \rightarrow \iy$, estimate $s$ to $\infty$. We will not prove the full claim as the proof is tedious and we judge that the given form is sufficient for making the point that DRFit behaves in a good way in the present setting. As it turns out it is more convenient to work in terms of the parameter $b=1/\alpha$ rather than $\alpha$ itself. Hence in the remainder of this section, $b$ will consistently denote $1/\alpha$.

Let $f_k$, $k=0,1$ stands for the density of $X_k$, each a mixture of $f^*_0$ and $f^*_1$ depending on the mislabelling distribution. Observe that according to Proposition \ref{prop:new_loss}, finding the DRFit optimiser of $s$ for a given $b$, means to do the optimisation

\begin{equation} \label{eq:weighted_logistic_estimator}
\hat{s}_w(b) = \text{argmax}\,\mathcal{O}_b(s)
\end{equation}
    where 
    \begin{align*} 
        \mathcal{O}_b(s) &= p_1\log\E\left[\left(\frac{e^{sX_1}}{e^{sX_1}+e^{-sX_1}}\right)^{b}\right] \\
        &+ 
        p_0\log\E\left[\left(\frac{e^{-sX_0}}{e^{sX_0}+e^{-sX_0}}\right)^{b}\right].
    \end{align*}

\begin{proposition} \label{prop:overestimate}
    Consider the optimisation problem (\ref{eq:weighted_logistic_estimator}).
    Assume that $f_1(x) \geq f_1(-x)$ and $f_0(x) \leq f_0(-x)$ whenever $x \geq 0$. Then $\hat{s}_w(b) = \iy$ whenever $b \geq 1$.
\end{proposition}

\begin{proof}
Finding the maximum of $\mathcal{O}_b(s)$ is equivalent to finding the maximum of $\mathcal{E}_b(s) = \exp(\mathcal{O}_b(s))$, i.e.\ the maximum of
\begin{align*}
    \mathcal{E}_b(s) &:=\E \left[ \left( \frac{e^{sX_1}}{e^{sX_1}+e^{-sX_1}}\right)^{b}\right]^{p_1} \\
    &\cdot\E\left[\left(\frac{e^{-sX_0}}{e^{sX_0}+e^{-sX_0}}\right)^{b}\right]^{p_0}.
\end{align*}

Now we claim that both factors of $\mathcal{E}_b(s)$ are increasing in $s$ when $b \geq 1$. To see that this is so, observe that for any $b \geq 0$,
\begin{align*}
    & \int_{0}^\iy \left(\frac{e^{sx}}{e^{sx}+e^{-sx}}\right)^b f_1(x) +
    \left(\frac{e^{-sx}}{e^{sx}+e^{-sx}}\right)^b f_1(-x)dx \\
    &= \int_0^\iy \left( \left(\frac{e^{sx}}{e^{sx}+e^{-sx}}\right)^b + \left(\frac{e^{-sx}}{e^{sx}+e^{-sx}}\right)^b\right)f_1(-x)dx \\
    &+ \int_0^\iy \left(\frac{e^{sx}}{e^{sx}+e^{-sx}}\right)^b (f_1(x)-f_1(-x))dx. 
\end{align*}
Since $f_1(x) \geq f_1(-x)$ by assumption, the second term is increasing in $s$. If also $b \geq 1$, the first term is nondecreasing in $s$ since the integrated function is then nondecreasing. This proves that the first factor of $\mathcal{E}(s)$ is increasing and the second factor follows analogously.

\end{proof}

A consequence of Proposition \ref{prop:overestimate} is that if it could be shown that $\hat{s}_w(b)$ is continuous in $b$ for $b \leq 1$ with $\lim_{b \uparrow 1}\hat{s}_w(b)=\iy$, would be that for some $b=b^*$ one has $\hat{s}_w(b^*)=\hat{s}^*$. However, $\mathcal{E}_b(s)$ is in general very difficult to analyse analytically, so it is hard to come up with explicit general conditions that ensure this. Numerically, we have verified that $\hat{s}_w$ is indeed continuous for numerous natural cases of distributions on $X^*_1$ and $X^*_0$. On the other hand, we have also found examples where continuity does in fact not hold. One such example is $\Pro(X^*_0=-1)=\Pro(X^*_0=1)=1/2$, $\Pro(X^*_1=-1)=\Pro(X^*_1=1)=1/10$, $\Pro(X^*_1=10)=4/5$ and 20\% mislabels for both classes. In this case we find that in the range $\alpha \in [1.54,1.69]$, two local maximums in $\mathcal{E}(s)$ appear and the the global maximum shifts from one to the other at approximately $\alpha=1.62$. (If one prefers, this example can be turned into one with continuous distribution on the covariate by adding some sufficiently small continuous random noise.)

However adding the strong symmetry assumptions that $X^*_0 =_d -X^*_1$, $p_1=p_0=1/2$ and $p^*_1=p^*_0=1/2$, maximising $\mathcal{E}_b(s)$ boils down to maximising the target function 
\[\E\left[\left(\frac{e^{sX_1}}{e^{sX_1}+e^{-sX_1}} \right)^{b}\right]
= \E\left[\frac{1}{(1+e^{-2sX_1})^{b}}\right]. \]
over $s$. This case can be handled more easily. Using in addition the same density condition as in Proposition \ref{prop:overestimate} gives the following.

\begin{proposition} \label{prop:uniqueness}
    Assume that $X^*_0 =_d -X^*_1$, $p^*_0=p^*_1=1/2$, $f_1^*(x)>f_1^*(-x)$ whenever $x>0$ and that mislabels occur independently and with equal probability $q=1-p<1/2$ for the two classes. Then $\hat{s}_w(b)$ is continuous increasing, $\lim_{b \uparrow 1} \hat{s}_w(b)=\iy$ and there is a $b^* > 1$ such that $\hat{s}_w(b^*) = \hat{s}^*$.
\end{proposition}


\begin{proof}
From Proposition \ref{prop:overestimate}, we know that $\hat{s}_w(b) = \iy$ for $b \geq 1$, so we may assume that $b < 1$.
Taking the derivative of the target function $\mathcal{E}_b(s)$, the proposition follows once we have proved that the solution in $s$ of 
\[M(b,s) := \E[X_1g(sX_1,b)] = \int xg(sx,b)f_1(x)dx\]
where $f_1(x)=pf_1^*(x)+qf_1^*(-x)$ and 
\[g(t,b)= \frac{e^{-2t}}{(1+e^{-2t})^{b+1}},\]
is increasing in $b$. First note that 
\[\int xg(sx;b) dx<0\]
for $b<1$, from which it readily follows that $M(b,s)<0$ for sufficiently large $s$. Hence there is exists at least one $s$ for which $M(b,s)=0$ (since the condition $f_1(x)>f_1(-x)$ for $x>0$ makes sure that $M(b,s)>0$ for sufficiently small $s$).

Assume that the support of $f_1$ is bounded, i.e.\ there are $-\iy<a_1<0<a_2<\iy$ such that $f_1(x)=0$ for $x \not\in [a_1,a_2]$. By scaling $X$, we can then without loss of generality assume that $-1 \leq a_1 < a_2 \leq 1$.
The function $M$ is continuously differentiable and
\[\frac12 M'_s(b,s) = \frac{\partial}{\partial s}M(b,s) = \int \frac{x^2e^{-2sx}(be^{-2sx}-1)}{(1+e^{-2sx})^{b+2}}f_1(x) dx.\]
Then
\[M'_s(b,s)-M(b,s) = \int \frac{((bx-1)e^{-2sx}-x-1)e^{-2sx}}{(1+e^{-2sx})^{b+2}}xf_1(x)dx
=: \int r(b,s,x)xf_1(x)dx.\]
It is readily checked that $r(b,s,x)+r(b,s,-x) < 0$ for all $(b,s)$ and $x \in [0,1]$ as
\[r(b,s,x)+r(b,s,-x) = \frac{bx-1-(x+1)e^{-2bsx}-(b+1)e^{4sx}-(1-x)e^{2sx}}{(1+e^{2sx})^{b+2}},\]
and $r(b,s,x)<0$ for all $x>0$. Hence, since $f_1(x)>f_1(-x)$ for $x>0$, $M'_s(b,s)<0$ for any $s$ with $M(b,s)=0$. By taking limits as $-a_1,a_2 \rightarrow \iy$, we can drop the assumption of bounded support and have in general that $M'_s(b,s) < 0$ whenever $M(b,s)<0$.

 The Implicit Function Theorem states that for each $(b_0,s_0)$ with $M(b_0,s_0)=0$, there is an open neighbourhood where $s$ is uniquely defined as a continuous function $s(b)$ of $b$ implicitly given by $M(b,s)=0$, provided that $M'_s(b_0,s_0) \neq 0$. Since the latter was just shown to be true for all $(b,s)$, it follows that $s(b)$ is uniquely defined for all $b$ and hence equals $\hat{s}_w(b)$, which is thus continuous.

Also, the implicit derivative of $s$ with respect to be is given by
\[\frac{\partial s}{\partial b}= - \frac{\partial M/\partial b}{\partial M/\partial s}.\]
We have just established that the denominator is negative. 
The numerator is
\[\frac{\partial M}{\partial b} = -\int \frac{x \log(1+e^{-2sx})}{(1+e^{-2sx})^{b+1}}f_1(x)dx,\]
which is positive as for $(b,s)$ with $M(b,s)=0$
\[\int \frac{xe^{-2sx}}{(1+e^{-2sx})^{b+1}}f_1(x) dx= 0,\]
and multiplying the integrated function by the decreasing function $\log(1+e^{-sx})$ hence results in a negative integral and hence a positive numerator. This gives
\[\frac{\partial s}{\partial b} = \frac{\partial \hat{s}_w}{\partial b}> 0.\]
So $\hat{s}_w(b)$ is increasing in $b$ and we the final claim now follows from Propositions \ref{prop:underestimate} and \ref{prop:overestimate} and the Intermediate Value Theorem.
\end{proof}

{\bf Remark on the proof.} An alternative argument that $M(b,s)=0$ cannot have more than one solution in terms of $s$ for a given $b$ is that if so, then at least one solution $s_0$ has to correspond to a local minimum of $\mathcal{E}_b$. Since $(\partial/\partial s)M(b,s)$ is the second derivative of $(1/b)\mathcal{E}_b$, it cannot be negative at $s_0$, which it was shown in the proof that it in fact is.

\medskip

A number of natural cases are covered by Proposition \ref{prop:uniqueness} and thus satisfy that $\hat{s}_w(\alpha)$ is decreasing and for some $\alpha^*>1$ it holds that $\hat{s}_w(\alpha^*)=\hat{s}^*$:

\begin{itemize}
    \item $X_1  \sim$ logistic with positive mean.
    \item $X_1  \sim$ Cauchy centred at a positive number.
    \item $X_1  \sim$ Gaussian with positive mean.
\end{itemize}

Next we move to the multivariate case $r \geq 2$. The above results show that with a single covariate, there is at least one choice of $b$ that improves on standard logistic regression and under some reasonable symmetry conditions a $b$ that completely cancels the effect of mislabelling. The arguments are one-dimensional in nature, and in particular rely on the Intermediate Value Theorem. In higher dimension, it is intuitively unlikely to be true that there is a choice of $b$ that would cancel the fact of mislabelling altogether, since the estimates $\hat{s}_w(b)$, which are now multi-dimensional, now describe a curve in $\mathbb{R}^r$ and cannot be expected to hit $\hat{s}^*$ exactly. However, as it turns out this is exactly what happens in a very significant special case, namely when the covariate vector $\bX^*_1$ is multivariate Gaussian. In this case, very strong results hold; one can determine a vector $\bu$, independent of $b$ and the mislabelling probability $q$, such that $\hat{\bs}(b) \propto \bu$ and $\hat{\bs}^*$ can be determined explicitly.

\begin{theorem} \label{thm:exactly_right_multivariate}
 Assume that the distribution of the covariates $\bX^*_k$ of an observation with true label $k$ are multivariate Gaussian such that $\bX^*_0 =_d -\bX^*_1$, that $\E[\bX^*_1]=\one$ and $p_1^*=p^*_0=1/2$. Assume also that the mislabelling probability $q=1-p$ is class independent. Let $\Sigma$ be the covariance matrix of $\bX^*_1$ and let the unit vector $\bu$ be given by
 \[\bm{u} = \frac{\Sigma^{-1} \one}{\sqrt{\one^T \Sigma^{-2} \one}}.\] Then there exists an increasing function $\hat{c}:[0,1) \rightarrow \RR$ such that $c(b) \rightarrow \iy$ as $b \uparrow 1$ and
 \begin{itemize}
     \item[(i)] $\hat{\bs}(b) = \hat{c}(b)\bm{u}$,
     \item[(ii)]$\hat{\bs}^* = \Sigma^{-1}\one$.
 \end{itemize}
 In particular $\{\hat{\bs}(b)\}_{b \in [0,1)}$ is a half-line with its finite endpoint in $\hat{c}(0)\bm{u}$ and there is a $b^* \in [0,1)$ such that $\hat{\bs}(b^*) = \hat{\bs}^*$.

\end{theorem}

\begin{proof}
The parameter estimates $\hat{\bs}$ are given by maximising $\E[(1+e^{-2\bs^T\bX_1})^{-b}]$, which means to solve the system of equations
\[\E[X^*_{1i}(pg(\bs^T\bX^*_1;b)-qg(-\bs^T\bX^*_1;b))]=0,\, i=1,\ldots,p\]
over $\bs$. Another way to phrase this is that $\hat{\bs}$ solves
\begin{equation} \label{eq:linear_combination_of_partial_derivatives}
\E[\bm{v}^T\bX^*_1 (pg(\bs^T\bX^*_1;b)-qg(-\bs^T\bX^*_1;b))]=0
\end{equation}
for all $\bm{v} \in \RR^r$. In particular this of course goes for $\bm{v} = \bs$, so
\begin{equation} \label{eq:one_dimensional_for_optimal_linear_combination}
\E[\bm{s}^T\bX^*_1 (pg(\bs^T\bX^*_1;b)-qg(-\bs^T\bX^*_1;b))]=0.
\end{equation}
Consider an $\bs$ that solves (\ref{eq:one_dimensional_for_optimal_linear_combination}), set $\bu$ to be the unit vector $\bs/||\bs||_2$.

For $\bv \in \RR^r$, $Cov(\bv\bX^*_1,\bu\bX^*_1) = \bv^T \Sigma \bu$ which is $0$ whenever $\bv$ is orthogonal to $\Sigma \bu$. Since $\bX^*_1$ is multivariate Gaussian, $\bv \bX^*_1$ is also independent of $\bu \bX^*_1$ for such $\bv$ and hence the following holds for all $\bv$ orthogonal to $\Sigma \bu$:
\begin{align} 
& \E[\bm{v}^T\bX^*_1 (pg(\bs^T\bX^*_1;b)-qg(-\bs^T\bX^*_1;b))] \nonumber \\
&= \E[\bm{v}^T\bX^*_1] \E[pg(\bs^T\bX^*_1;b)-qg(-\bs^T\bX^*_1;b)] =0. \label{eq:gradient_condition_for_independent_linear_combinations}
\end{align}
For simplicity, write for now $Y=\bs^T\bX^*_1$. 
We claim now that is is not possible to have $\E[Y(pg(Y;b)-qg(-Y;b))]=0$ and $\E[pg(Y;b)-qg(-Y;b)]=0$ at the same time. To see that, note that $g(-t;b)/g(t;b) = e^{2(1-b)t}$ from which it follows that for $b<1$, $pg(t;b)-qg(-t;b)$ has exactly one zero $t=t_0=\log(p/q)/2(1-b)$. Also, $pg(t;b)-qg(-t;b)$ is positive to the left and negative to the right of $t_0$. Now if the two expectations were both zero, then we would also have
\[\E[(Y-t_0)(pg(Y;b)-qg(-Y;b))] = 0.\]
However the random variable in the expectation is strictly negative except at $t_0$, a contradiction.

Having proved this claim, it now follows that whenever $\bv$ is orthogonal to $\Sigma \bu$, then $\E[\bv^T\bX^*_1]=0$. Since $\E[\bX^*_1]=\one$, it follows that such a $\bv$ must be orthogonal also to $\one$. 
Now let $\bv_1,\ldots,\bv_{r-1}$ be pairwise orthogonal and orthogonal to $\Sigma \bu$. Since these vectors must also be orthogonal to $\one$, it follows that $\Sigma \bu$ and $\one$ must align, i.e.\ $\bu=\Sigma^{-1}\one/||\Sigma^{-1}\one||_2$.

Now solving the remaining equation (\ref{eq:one_dimensional_for_optimal_linear_combination}) becomes to solve
\[\E[U(pg(cU;b)-qg(-cU;b))]=0]\]
for $c$, where $U=\bm{u}^T\bX^*_1$, a Gaussian variable with expectation $\bm{u}^T\one = \one^T \Sigma^{-1} \one> 0$. By Proposition \ref{prop:uniqueness}, the solution $\hat{c}(b)$ is unique and $\hat{s}(b)=\hat{c}(b)\bm{u}$.

Moreover, $\bm{u} \propto \one$ is independently of both $b$ and $p$. The first of these means that it follows from Proposition \ref{prop:uniqueness} that $\hat{c}(b)$ is increasing in $b$ and that $\{\hat{s}(b)\}_{b \in [0,1)}$ is the half-line $\{a\bm{u}: a \geq \hat{c}(0)\}$. The independence of $p$ implies that the optimum $\hat{\bs}^*$ for clean data is also in the direction of $\bm{u}$, i.e.\ $\hat{\bs}^* = a^*\bm{u}$ for some $a^*>0$, as this is the estimate corresponding to $q=b=0$. 
Applying Proposition \ref{prop:uniqueness} once again shows that $a^* \geq c(0)$, i.e.\ there exists a $b \in [0,1)$ such that $\hat{\bs}(b)=\hat{\bs}^*$. 
However, the theorem claims in addition that $\hat{\bs}^*$ exactly equals $\Sigma^{-1}\one$. That follows if $c = ||\Sigma^{-1}\one||_2$ satisfies $\E[Ug(cU;0)]=0$, in other words if, with $W=(\Sigma^{-1}\one)^T\bX^*_1$,
\[\E\left[W\frac{e^{-W}}{e^{W}+e^{-W}}\right] = 0.\]
It is readily computed that the mean and variance of $W$ are equal, namely they are both $\mu:=\one^T\Sigma^{-1}\one$. Hence
\begin{align*}
    \E\left[W\frac{e^{-W}}{e^{W}+e^{-W}}\right] &\propto
    \int w\frac{e^{-w}e^{-\frac{(w-\mu)^2}{2\mu}}}{e^w+e^{-w}}dw \\
    &\propto \int w\frac{e^{-\frac{w^2}{2\mu}}}{e^w+e^{-w}}dw = 0,
\end{align*}
where the last equality follows on observing that the integrated function is odd.
\end{proof}

{\bf Remark.} If one drops the scaling of covariates, the results generalise easily. Taking $\E[\bX^*_i] = \bm{\mu}$, then for clean data, one gets $\hat{\bs}^* = \Sigma^{-1}\bm{\mu}$, the estimates $\hat{\bs}(b)$ are on the line spanned by $\hat{\bs}^*$ and for some $b^*$, one has $\hat{\bs}(b^*) = \hat{\bs}^*$.

{\bf Remark.} There are only two places in the proof where the assumption of Gaussian covariates is used. One is implicitly used when applying Proposition \ref{prop:uniqueness} to $U$ as the assumption then guarantees that $U$ satisfies the density condition $f^*_1(x)>f^*_1(-x)$, but this condition would also be satisfied with many other covariate distribution assumptions. The other and more crucial one is to deduce the factorisation of the expectation in (\ref{eq:gradient_condition_for_independent_linear_combinations}), which fails since independence does not follow from zero correlation for non-Gaussian multivariate random variables. However, we believe that the required zero correlation between $\bv_i^T\bX^*_1$ and $g(\bs^T\bX^*_1;b)$ for the factorisation is not ``far off'' in most cases, in particular not if $r$ is large.

\smallskip

We have made a numerical computation with $r=2$ and with $X^*_{11}$ and $X^*_{12}$ independent and uniform on $[-1,3]$ and $[-2,4]$ respectively. It turns out that $\hat{\bs}^*$ and the $\hat{\bs}(b)$ closest to it are indeed very close. While the the unweighted estimator $\hat{\bs}(0)$ has $||\hat{\bs}(b^*)-\hat{\bs}^*||_2 / ||\hat{\bs}(0)-\hat{\bs}^*||_2 \approx 0.064$, for the optimal $b= b^* \approx 0.54$. We also tried the case with $X^*_{11}$ uniform on $[-1,3]$ and $X^*_{12}$ uniform on $[-1/4,5/4]$, i.e.\ with a bit more qualitative difference between the two coefficient distributions. We found that in this case, the optimal $b$ still produces an estimate that is a very large improvement: $||\hat{\bs}(b^*)-\hat{\bs}^*||_2 / ||\hat{\bs}(0)-\hat{\bs}^*||_2 \approx 0.046$. It would be interesting if one could come up with general bounds on $||\hat{\bs}(b^*)-\hat{\bs}^*||_2 / ||\hat{\bs}(0)-\hat{\bs}^*||_2$, but we leave that as an open problem.

\section{Numerical solver for DRFit}

In the general DRFit setting (i.e.\ with a penalty $g$ potentially different from the one given in Section \ref{Choice_of_weight_regularisation_function}), one cannot solve explicitly for $\omega$ in terms of $\theta$ in (\ref{ref:min}) and then has to settle for purely numerical methods. In that case, one has to alternate between applying numerical fitting of $\theta$ given $\omega$ and vice versa. This is to say that we split the optimisation into two sub-tasks:

\begin{align*}
    \text{Task \: 1: \: } \hat{\theta} = \underset{\theta}{\mathrm{\text{argmin}}} \sum_i \omega_i L(z_i,\theta) +  \alpha g(\omega) +  \frac{\lambda}{2} ||\theta||_2^2,
\end{align*}
\vspace{-.5cm}
\begin{align*}
    \text{Task \: 2: \: } \hat{\omega} = \underset{\omega, \omega \geq 0}{\mathrm{\text{argmin}}} \sum_i \omega_i L(z_i,\theta) + \alpha g(\omega) + \frac{\lambda}{2} ||\theta||_2^2.
\end{align*}

One alternates between these two tasks during training. Task $1$ is a standard minimisation problem for the network model parameters $\theta$.
The optimisation of $\omega$ is done with gradient decent which gives us the updating procedure 
$$\omega_i \leftarrow w_i - \beta[L(z_i,\theta) + \alpha g'(\omega)],$$ 
where $\beta$ is the learning rate for the observation weights. In order to speed up convergence, we use a burn-in period where we do not update the weights $\omega$. In this way, the model learns the more global structure of the data before we update the weights. In our experiments we found that a short burn-in of just a few epochs suffices.
In order to prevent the weights from becoming negative during training (something that can happen for the purely numeric algorithm), we use weight clipping to set negative weights to zero. 
In each iteration the weights are also normalised to sum to $\rho_c n$. This is to preserve the same learning rate throughout training. The training algorithm is summarised as Algorithm $\ref{algorithm:cost_weights}$.

\begin{algorithm}
\caption{Training algorithm with observation weights.}
\begin{algorithmic}
\Procedure{Train}{Training set $D$, $\alpha$,burn\_in,update\_frequency, $\rho_c$, $\beta, N$}

\State $\omega \gets \mathbf{1}$
\For{epoch = $0 \dots N-1$}
    \For{Batch $S \subseteq D$}
        \State Update $\theta$ with batch $S$ for current $\omega$.
        \If{(\text{epoch} $>$ \: \text{burn\_in}) \\ \hspace*{5.1em} \textbf{and}  \: (\text{mod}(\text{epoch},\text{update\_frequency}) \hspace*{0.5em}$ = 0$)}
            \State $\omega_S \gets \omega_S - \beta[L(S;\theta) + \alpha \nabla g(\omega_S)]$ 
            \For{i in S}
                \If{$(\omega_i < 0)$}
                    \State $\omega_i \gets 0$
                \EndIf
            \EndFor
            \State Normalise $\omega_S$ to mean $\rho_c$ for each class present in $S$
        \EndIf
    \EndFor
\EndFor
\Return
\EndProcedure
\end{algorithmic}
\label{algorithm:cost_weights}
\end{algorithm}

\section{Experiments}

We experimentally evaluate the performance of the DRFit method by training on different neural networks for classification tasks of varying difficulty. We train DRFit both (partly) analytically, in the sense of solving for $\omega$ in terms of $\theta$, and fully numerically according to Algorithm \ref{algorithm:cost_weights}. The reason for doing the latter is to see if there appears to be non-convergence phenomena that cannot be seen for the analytic solver. In each simulation, if not stated otherwise, we have used $\rho_c = p(c)/(1 - q(c))$ where $p(c)$ is the proportion among those instances that have been classified as class $c$ that are indeed of class $c$ and $q(c)$ is the proportion of instances for which the correct label is $c$, but have been classified as something else.
Good estimates of $\rho_c$ can be obtained from a set of well annotated validation data that only needs to be large enough that the proportions involved can be estimated with a decent precision. This scaling is used both for the DRFit method and standard regularised models.
Since one does not always have access to sufficient validation data to estimate the $\rho_c$:s, we include some experiments results when this scaling is omitted
(i.e.\ all $\rho_c$ are set to $1$) whose results are presented in the appendix. In all experiments a small burn-in period is used where we trained without training the observation weights.

The experiments are the following;

\begin{itemize}
    \item A small network on the MNIST dataset to classify 1's from 7's. The network has been chosen small enough that the number of parameters is small compared to the number of training examples: one hidden layer with eight nodes with ReLU activation, followed by a final logistic layer. Hence the network is not overparametrised per se and only the mislabels will cause it to overfit. We have hence here dropped the ridge regularisation term and can study the refined effect of observation weighting in a standard network.
    
    \item A subset of the well-known CIFAR-10 dataset. We have chosen to work with classifying cars vs aeroplanes. Here we use a convolutional neural network with three convolutional layers with $8$, $16$ and $32$ filters, max pooling between each layer and ReLu activations. This is followed by three dense layers with $124$, $64$ and $2$ neurons. The loss function used is standard cross-entropy loss. 
    
    \item The whole CIFAR-10 dataset. Here we have used the convolutional layers of the famous VGG16 network as preprocessing to extract features sent to a feed forward network with four layers with $64$, $48$, $32$ and $10$ neurons. Between each layer, we have used batch normalisation and ReLU activation. This gives us in total $1,611,002$ trainable parameters. 
     
\end{itemize}

For the first two datasets, we have considered three levels of distribution of mislabels: (i) 20\% mislabels in each class, (ii) 30\% mislabels in one class and 10\% in the other and (iii) 40\% mislabels in one class and none in the other. 
We have also compared with classification without noise. 
In the third dataset we have used $40 \%$ label noise in each class.

\subsection*{A small neural net on MNIST}

In order to distil the effect of the introduction of observation weights, we have chosen to train a minimalistic network on a contaminated subset of the MNIST dataset. The network in question was taken small enough that overfitting via overparametrisation in itself is not a problem and we have consequently dropped the ridge penalty term from the target loss functions. We have selected two classes, images of ones and sevens, and then sampled down the images to a size of $14 \times 14$. The network we used has one hidden layer with eight nodes followed by a final logistic layer. We introduced random label noise where we flipped labels from ones to sevens and sevens to ones in the training set according to the proportions specified. The hyperparameter $\alpha$ was optimised on an separate validation set. We then executed $100$ training runs with the optimal $\alpha$ and compared this with $100$ runs for the same network without observation weights. In Figure $\ref{fig:mnist_small_net_test_data}$, the results are presented. We trained DRFit both by Algorithm \ref{algorithm:cost_weights} and by the partly analytic solver for the observation weights and compared with the same network network with no regularisation. It is clearly seen that the unregularised network suffers heavily from overfitting while DRFit does not have this problem. 

In Figure $\ref{fig:histogram_of_weights_for_mnist_small}$ we can see that when we have $30 \%$ noise in ones and $10 \%$ noise in sevens, we almost get a perfect separation between the observation weights between the correctly classified data-points and the mislabelled ones. This explains how DRFit has virtually no  problems at all with overfitting; DRfit learns the false labels early on and disregards them completely from then on and since the model is not overparametrised in itself, the test accuracy remains optimal. Some of the few mislabelled images that DRFit could not detect can be seen in Figure \ref{fig:false_positives} along with some correctly classified images that DRFit considered mislabelled.

\begin{figure}[H]
    \centering
    \includegraphics[width=1\textwidth]{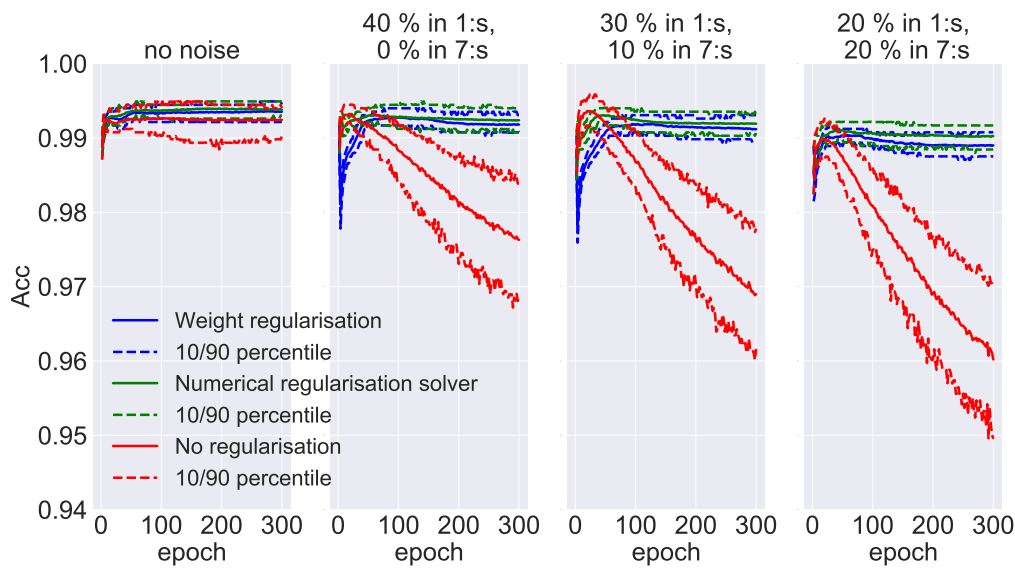}
    \caption{The mean accuracy on test data during training of $100$ different random initialisation settings, done with DRFit both with a numerical and a partly analytic solver and the same network without regularisation. This is done in four different label noise settings. With no label noise, with $40 \%$ noise in class 1 and $0 \%$ noise in class 2, $30 \%$ noise in class 1 and $10 \%$ noise in class 2 and with $20 \%$ noise in class 1 and $20 \%$ noise in class 2}
    \label{fig:mnist_small_net_test_data}
\end{figure}

\begin{figure}[H]
    \centering
    \includegraphics[height=5cm, width=1\textwidth]{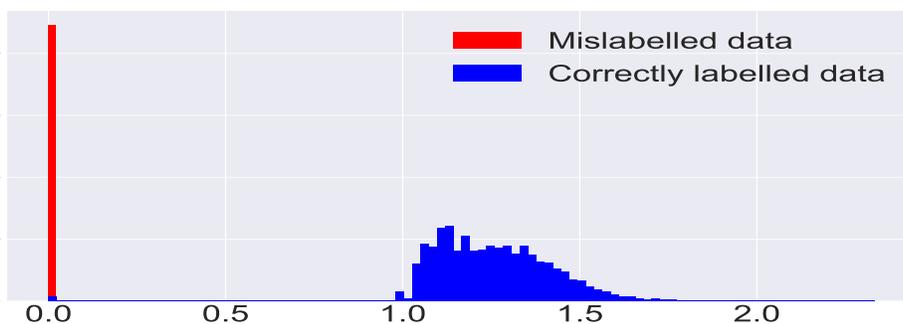}
    
    \caption{Histograms for the observation weight distributions for correctly labelled data and for mislabelled data. These are produced by training the DRFit on the subset of ones and sevens in the MNIST dataset.}
    \label{fig:histogram_of_weights_for_mnist_small}
\end{figure}

\begin{figure}[H]
    \centering
    \includegraphics[height=5cm, width=1\textwidth]{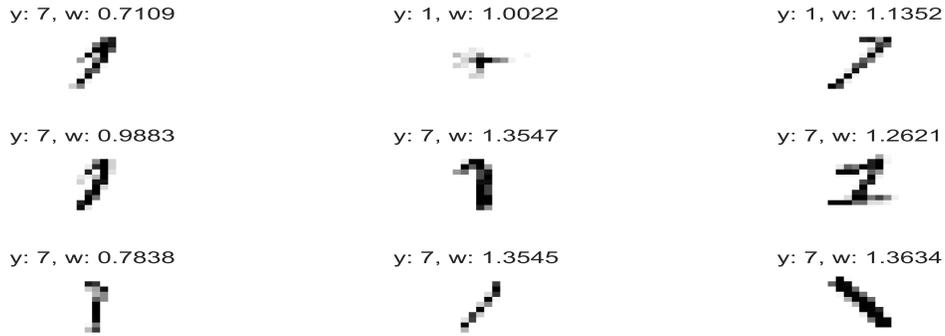}
    \caption{Example of images that where mislabelled but the DRFit method did not detect. Above of each image is the label the model was given and the weight the model set to the data point.}
    \label{fig:false_positives}
\end{figure}

\begin{figure}[H]
    \centering
    \includegraphics[height=5cm, width=1\textwidth]{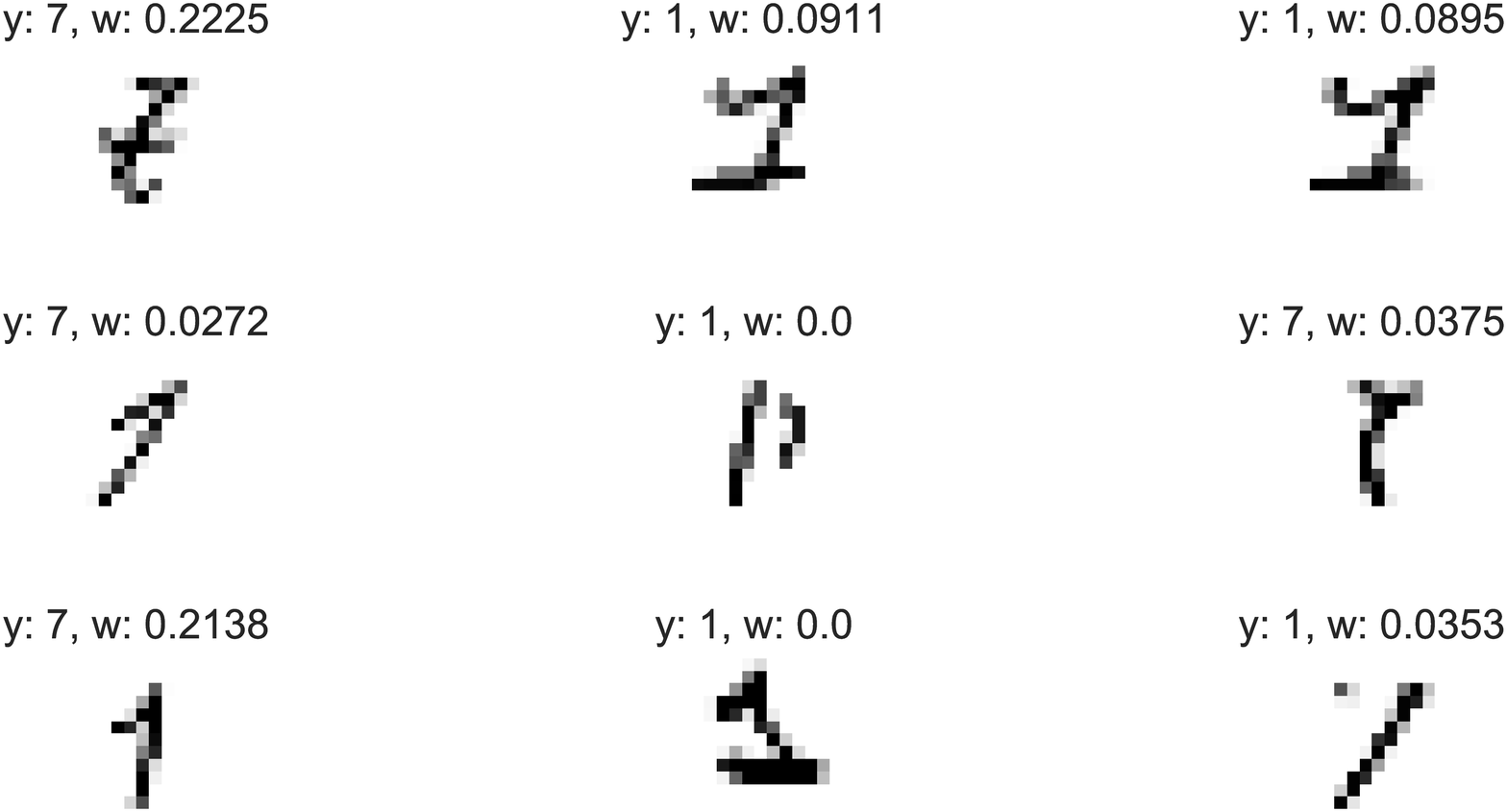}
    \caption{Example of images that are classified correctly but have been assigned a low observation weight by DRFit. Above each example is the label the model was given and the weight the model gave the data point}
    \label{fig:false_negatives}
\end{figure}

\begin{figure}[H]
    \centering
    \includegraphics[height=5cm, width=1\textwidth]{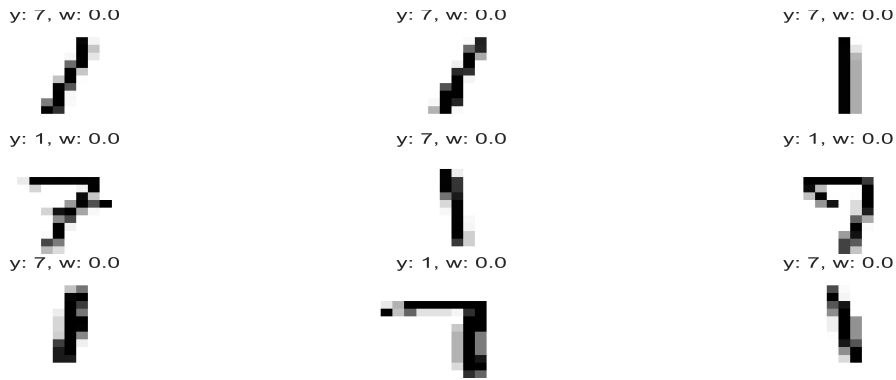}
    
    \caption{Example of images that were mislabelled and that DRFit could detect. Above each example is the label the model was given and the weight that DRFit assigned to the data point.}
    \label{fig:True_positive}
\end{figure}

\subsection*{A deep convolutional neural net on CIFAR-10}

After finding the optimal parameters $\lambda$ and $\alpha$ on a separate validation set, we trained the network $100$ times with different random parameter initialisations. We found that using ridge regularisation only is highly sensitive to initialisation and frequently crashes (i.e.\ validation accuracy soon drops to $50 \%$ accuracy and stays there for the rest of the training period). To make a direct comparison with DRFit we therefore remove all crashed runs from both DRFit and ridge. In these experiments, DRFit is by far more robust to parameter initialisation. Indeed, DRFit did not crash for any of the 100 runs, whereas ridge for some noise settings crashed up to nearly 40 \% of the runs and even 9 \% of the runs without any label noise.

In Table $\ref{tab:results_cifar10}$ we can see the corresponding hyperparameters and the final mean accuracy on test data for the different noise settings. 
Notice here that with a more uneven noise in the data, $\lambda$ tends to increase leading to a stronger regularisation with respect to the network parameters. 
No such trend can be observed in the $\alpha$ parameter. However, from our experience, the model performance tends to be relatively insensitive to variations on $\alpha$, meaning that there could be lots of noise present in the $\alpha$ optimisation. As intuition says, we can see that $\alpha$ is higher when no noise is present compared to the more balanced noise settings due to the benefit of including all data in the training set. 
In figure \ref{fig:runs_on_test_data_cifar10}, we can see that in the majority of the noise settings, the use of observation weights decreases overfitting. The optimal peak in accuracy, however, seems to be unchanged between the compared models.
While the ridge regularised network needs an extensive test data set to find a good stopping criterion, the DRFit network only needs an accurate estimation of $\rho_c$. (However, investigating the sensitivity with respect to the $\rho_c$ parameters will be saved for future research.)  

In Figure $\ref{fig:weights_hist_cifar10}$ we illustrate how the data weights are distributed for correct and mislabelled data, respectively, in the 30/10-noise setting. Here we notice that DRFit did an excellent job of identifying the mislabelled data, albeit a small number of mislabelled observations were not found and a few correctly labelled observations were erroneously found as mislabelled.
Notice that $\alpha$ works as a scaling parameter for the weights. This means that with a lower $\alpha$, we get a larger separation between correct and mislabelled data at the cost of having more correct data with lower weights; clearly too low an $\alpha$ will result in too much information being lost. By judging data points with observation weights below a fixed threshold as mislabelled, choosing this threshold will result in tuning the balance between the proportion of correct annotations rightly judged as correct and the proportion of mislabelled examples rightly judged as mislabelled. Figure \ref{fig:roc_weights} illustrates this. Notice e.g.\ that, when using optimal hyperparameters, one can find thresholds such that, when using the mean weight of data points over the $100$ runs to compare with the threshold, more than 90\% of correct labels are judged as correct and more than 90\% of false labels are considered false.

\begin{table}[h!]
    \centering
    \begin{tabular}{c|c|c|c|c|c|c}
        class 1 (\%) & class 2 ($\%$) & Opt. $\lambda$ & Opt. $\alpha$ & Acc test end& Acc test top &$\%$ crashes removed\\
        \hline
        \multicolumn{6}{c}{DRFit} \\
        \hline
        $0$ & $0$ & $0.10281$ & $3.0309$ & $0.94734$ &   $0.95409$&$0$\\
        $20$ & $20$ & $0.059512$ & $1.1712$ & $0.90177$ &$0.91854$&$0$\\
        $30$ & $10$ & $0.30801$ & $1.1222$ & $0.90878$ & $0.91911$&$0$\\
        $40$ & $0$ & $0.56716$ & $3.1629$ & $0.91544$ &  $0.92816$&$0$\\
        \hline
        \multicolumn{6}{c}{RidgeNet} \\
        \hline
        $0$ & $0$ & $0.50868$ & - & $0.94773$ & $0.95507$&$9$\\
        $20$ & $20$ & $0.16716$ & - & $0.82947$ & $0.91690$&$38$\\
        $30$ & $10$ & $0.26122$ & - & $0.83802$ & $0.91790$&$38$\\
        $40$ & $0$ & $0.23280$ & - & $0.82148$ & $0.92814$&$24$
    \end{tabular}
    \caption{Table showing optimal values of the hyperparameters $\lambda$ and $\alpha$ for different noise settings, mean accuracy at the end of training, the top accuracy during training and the percentage of crashes of the runs. Accuracy for both models are calculated on test data over the $100$ training runs with these optimal hyperparameter values.}
    \label{tab:results_cifar10}
\end{table}

\begin{figure}[H]
    \centering
    \includegraphics[height=4.6cm, width=1\textwidth]{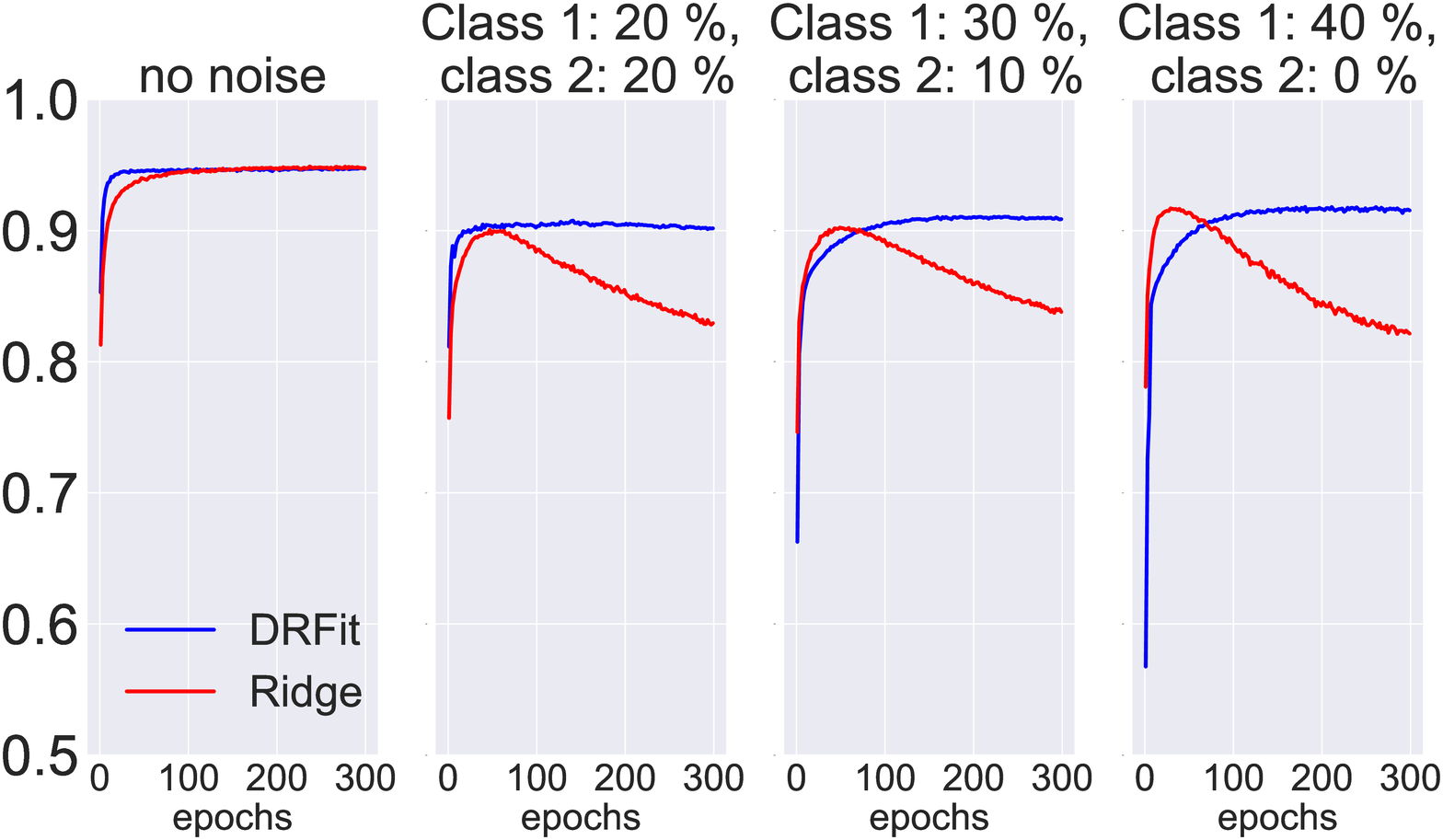}
    \caption{The average accuracy on a test data for DRFit and pure ridge regularisation during training when trained on the classes cars and planes in the CIFAR-10 dataset. In order to vary the noise we have done this in four different label noise settings. One with equally amount of noise in both classes, one with $30 \%$ noise in class one and $10 \%$ in class two, one with $40 \%$ noise in class one and no noise in class two and one setting with no label noise in the training data. In order to get a better comparison the runs that crashed are removed (see Table 1).}
    \label{fig:runs_on_test_data_cifar10}
\end{figure}

\begin{figure}[H]
    \centering
    \includegraphics[height=5cm, width= 1\textwidth]{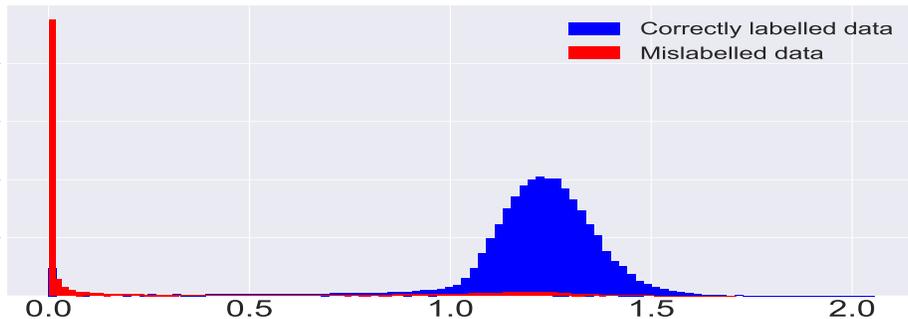}
    \caption{Histograms of the observation weight distributions over all simulations for mislabelled and correctly annotated data respectively for the 30/10-contaminated data set. The runs that crashed are removed.}
    \label{fig:weights_hist_cifar10}
\end{figure}

\begin{figure}[H]
    \centering
    \includegraphics[height=5cm, width=1\textwidth]{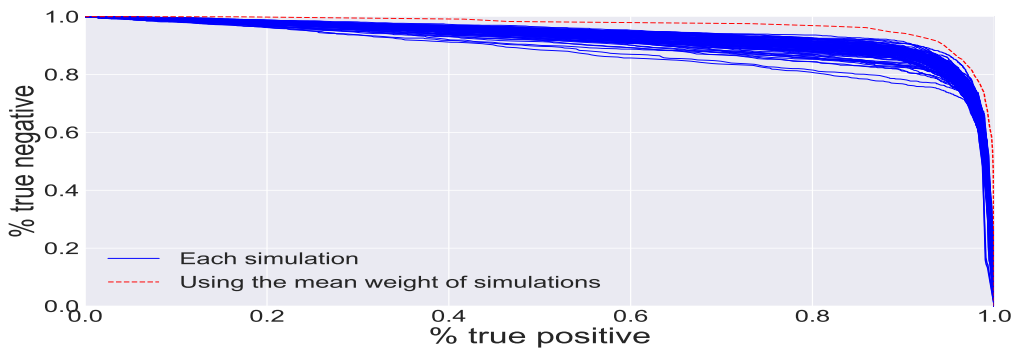}
    \caption{Curves that describe the balance between finding mislabelled points versus finding correctly labelled points while tuning a threshold $t$. For each value $t$ the model classifies a data point as mislabelled if its observation weight $\omega_i < t$ and as correctly labelled if $\omega_i > t$. Each blue curve is for one training run and the red curve is the result of using the average $\omega_i$ over the training runs to compare with $t$.}
    \label{fig:roc_weights}
\end{figure}

\subsection{A transfer learning model on CIFAR-10}

In this section, we consider classification of all ten classes in the CIFAR-10 dataset with $40 \%$ label noise in each class, i.e. $\rho_c = 1$ for each class. As preprocessing of the data, we used the top layers of the famous VGG16 network for feature extraction and then trained a smaller dense network with three layers for classification. We compare DRFit, regular $L2$ penalisation (i.e.\ ridge) and the Mixup method as presented in \citep{zhang2017mixup}. After finding the optimal hyper-parameters on a separate validation set and then retraining, the models are evaluated on a test set. The optimal hyperparameters, validation and test accuracy can be seen in Table $\ref{tab:parameters_transfer}$. In Figure $\ref{fig:test_acc_transfer}$ we can see how the test accuracy changes during training. Here, after some smoothing, we see that DRFit gives better performance
than regular Ridge regularisation and Mixup. In Figure $\ref{fig:weights_transfer}$ we see how the weights are distributed between correctly labelled data and mislabelled data and that in also this case, the DRFit can locate the mislabelled data quite well. 

\begin{table}[H]
    \centering
    \label{tab:transfer_learning_test_data}
    \begin{tabular}{c|c|c|c|c|c}
        Method & $\lambda_{opt}$ & $\alpha_{opt}$ & train acc &val acc & test acc \\
        \hline
        DRFit & $0.09$ & $1$ &$0.5426$ &$0.7942$ & $0.8041$\\
        Ridge & $0.01$ & - & $0.4978$ &$0.7668$ & $0.7977$ \\
        Mixup & $0.02$ & $0.2$ & $0.4440$ & $0.7704$ & $0.7662$ \\
    \end{tabular}
    \caption{Table showing the optimal hyper parameters, validation accuracy and test accuracy when training on CIFAR-10 with a transfer learning model. The validation and test accuracy is calculated as an average between five runs.}
    \label{tab:parameters_transfer}
\end{table}

\begin{figure}[H]
    \centering
    \includegraphics[height=5cm,width=1\textwidth]{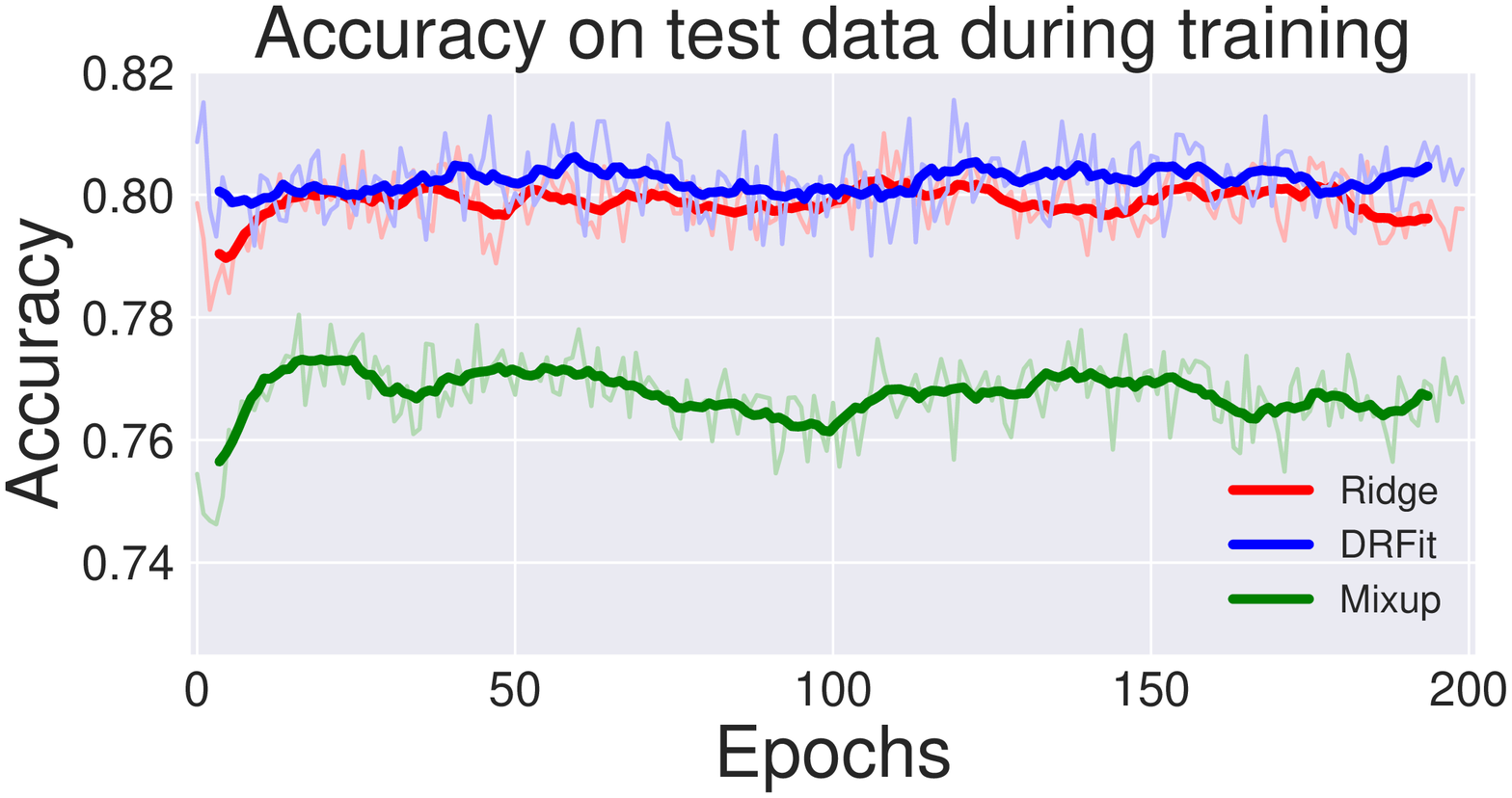}
    \caption{The test accuracy during training with optimal hyper-parameters. The optimal parameters are tuned on a separate validation set. For clarification, the results are smoothed with an average filter with the length of $10$ epochs. The raw data is represented with transparent lines.}
    \label{fig:test_acc_transfer}
\end{figure}

\begin{figure}[H]
    \centering
    \includegraphics[height=5cm,width=1\textwidth]{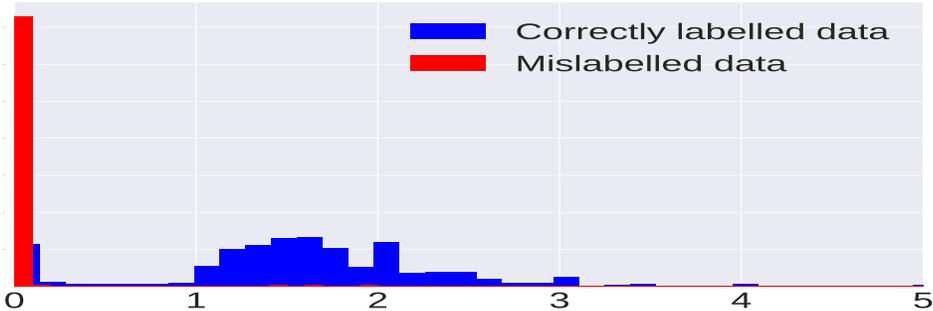}
    \caption{The distributions of weights produced from DRFit of correcly labelled and mislabelled data.}
    \label{fig:weights_transfer}
\end{figure}

Note that all models achieve high accuracy early on during training and show no sign of overfitting. This is probably due to the transfer model that is not tuned during training and already produces good features for this problem. However, since DRFit can identify mislabelled data very well (as seen in figure $\ref{fig:weights_transfer}$), the model is equipped for better generalisation.

\section{Discussion}

We have proposed a double regularisation framework, DRFit, for the training of a predictive model in the presence of mislabelled training data. We have presented a theoretical result that supports the method and we have experimentally demonstrated that DRFit improves performance over standard regularisation in three experiments on different data sets with different neural net predictive models and different noise distributions. 

We found that combining regularisation on observation weights with regularisation on model parameters (or with no regularisation when the model is not overspecified) results in higher, or at least as high, classification accuracy and significantly stronger robustness against overfitting.

In addition to test accuracy, we found, both on MNIST and CIFAR-10, that DRFit performs remarkably well for separating mislabelled training examples from those with correct labels. In the MNIST case, we could also clearly see that mislabelled training examples that DRFit failed to detect were in fact very ambiguous.

Yet another benefit of DRFit is that early stopping does not have to be explicitly applied; we can run training until convergence. By contrast, standard neural networks with pure parameter regularisation rely to a large extent on early stopping in the presence of label noise. (As shown in \cite{pmlr-v108-li20j}, early stopping will indeed make the model more robust to label noise, which is a result we also have seen in our experiments). 
DRFit is thus a more intelligible model in that we actually reach a minimum of the chosen loss minimisation objective. 
We strongly believe that the reason that DRFit does not overfit is the effect that was intended, namely that mislabelled data is to a very large extent "turned off" early on in the training process and that this works as a replacement for early stopping. This means that we get an optimisation problem that we actually optimise in contrast to the early stopping case. For early stopping to work effectively, one needs an extensive test data set to get a good criterion for stopping. In the DRFit case, one only needs the global parameters $\rho_c$ from the validation data for the method to work, potentially reducing the amount of data needed.  

The experiments on both datasets show that using observation weights does not have any adverse effects when training on a dataset without any mislabelled data. Moreover, in the MNIST case we see that DRFit is more consistent in the sense that the variation in test accuracy is smaller than for the nonregularised network. 

The beneficial effect of double regularisation seems to be pronounced in all tested noise distributions. While the inclusion of the scaling factors $\rho_c$ improved performance overall, DRFit still outperformed standard regularisation when $\rho_c=1$ was used (see the appendix). In practice, we can obtain estimates of the noise level for the different classes from the validation set. Alternatively, as a two-stage training alternative, one could estimate the noise levels from weight distributions from initial training with no scaling. This latter option has been reserved for future work.

An unavoidable drawback of double regularisation compared to standard ridge or lasso, is the need for two hyperparameters instead on one and hence optimisation of the them is a more costly process. However experiments indicate that DRFit is fairly insensitive to reasonably small changes in $\alpha$ and our optimisation algorithm was coarse-grained and far from a complete grid search. We also did not optimise the burn-in period or the learning rate for the observation weights in Algorithm \ref{algorithm:cost_weights}, neither during hyperparameter optimisation, nor during the test runs.
This means that DRFit could potentially perform 
better if further fine tuning in the hyperparameter optimisation were applied. The optimisation of these parameters is a non-trivial problem which we leave for future research.

In this paper, we restricted ourselves to binary classification and multiclass classification with balanced noise. However the general formulation of DRFit is by no means restricted to that, or even to classification problems. We expect, however, that in cases of class imbalance in noise, the case of multiclass classification needs more care.

Future research along the lines presented can analyse other types of double regularisation. This will potentially lead to new interesting findings on regularisation synergies. Another obviously important extension is to go beyond the assumption of random noise. While random noise is a natural assumption in some situations, there are clearly situations where noise may be more pronounced for training examples that are "close" to other classes in feature space and yet other situations where it is natural to assume adversarial noise.

\newpage
\bibliography{sample}

\newpage
\appendix
\section*{Appendix}

\begin{figure}[H]
    \centering
    \includegraphics[height=5cm, width=1\textwidth]{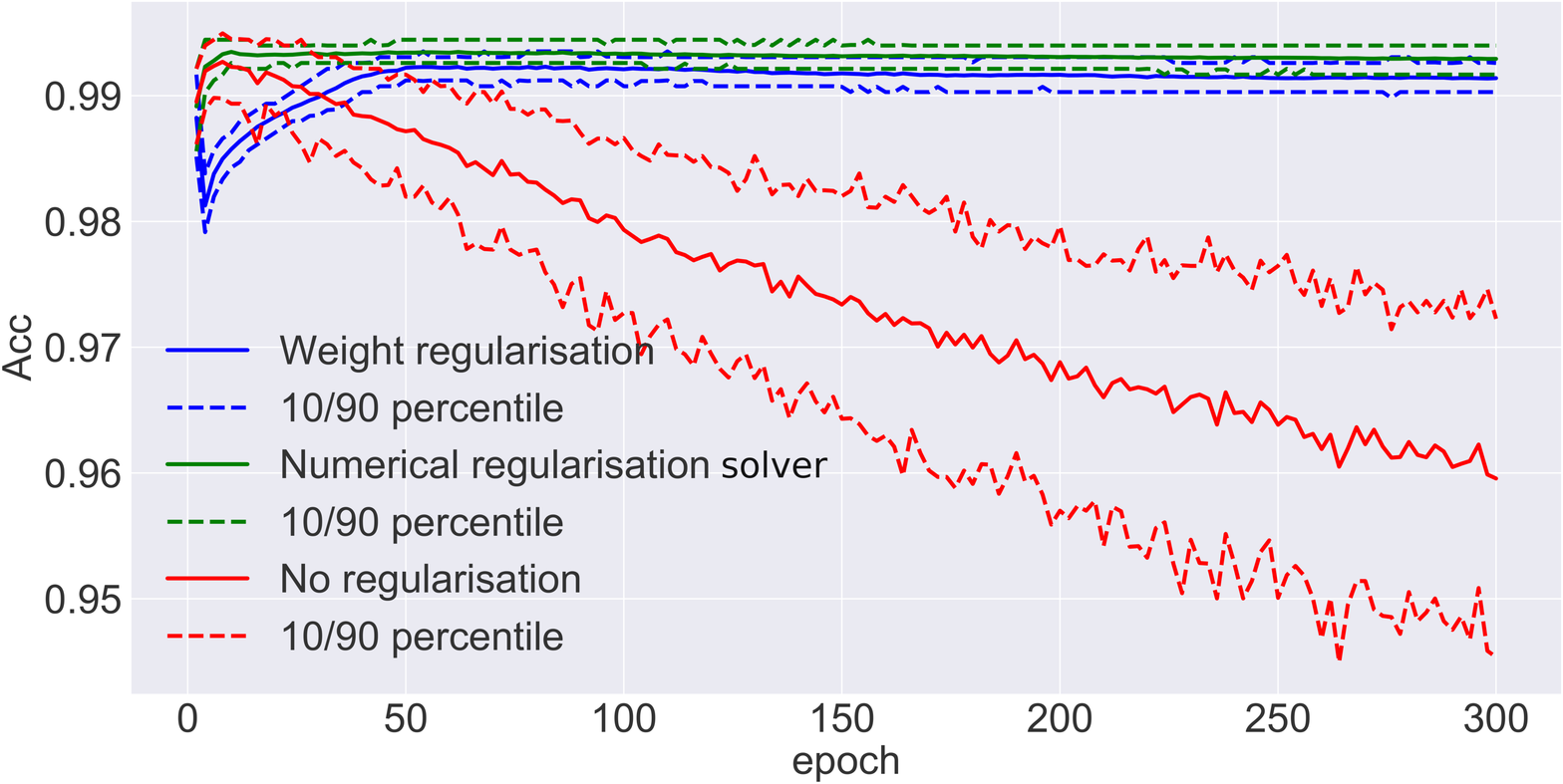}
    \caption{The mean accuracy on test data during training of $100$ different random initialisations, done with DRFit both with a numerical and a analytical solver and the same network without regularisation. This is done in with $30 \%$ label noise in class ones and $10 \%$ noise in class sevens. In this simulation we have assumed that we do not know anything about the noise distribution so $\rho_k = 1$.}
    \label{fig:mnist_small_net_test_data_appendix}
\end{figure}

\begin{figure}[H]
    \centering
    \includegraphics[height=4.6cm, width=1\textwidth]{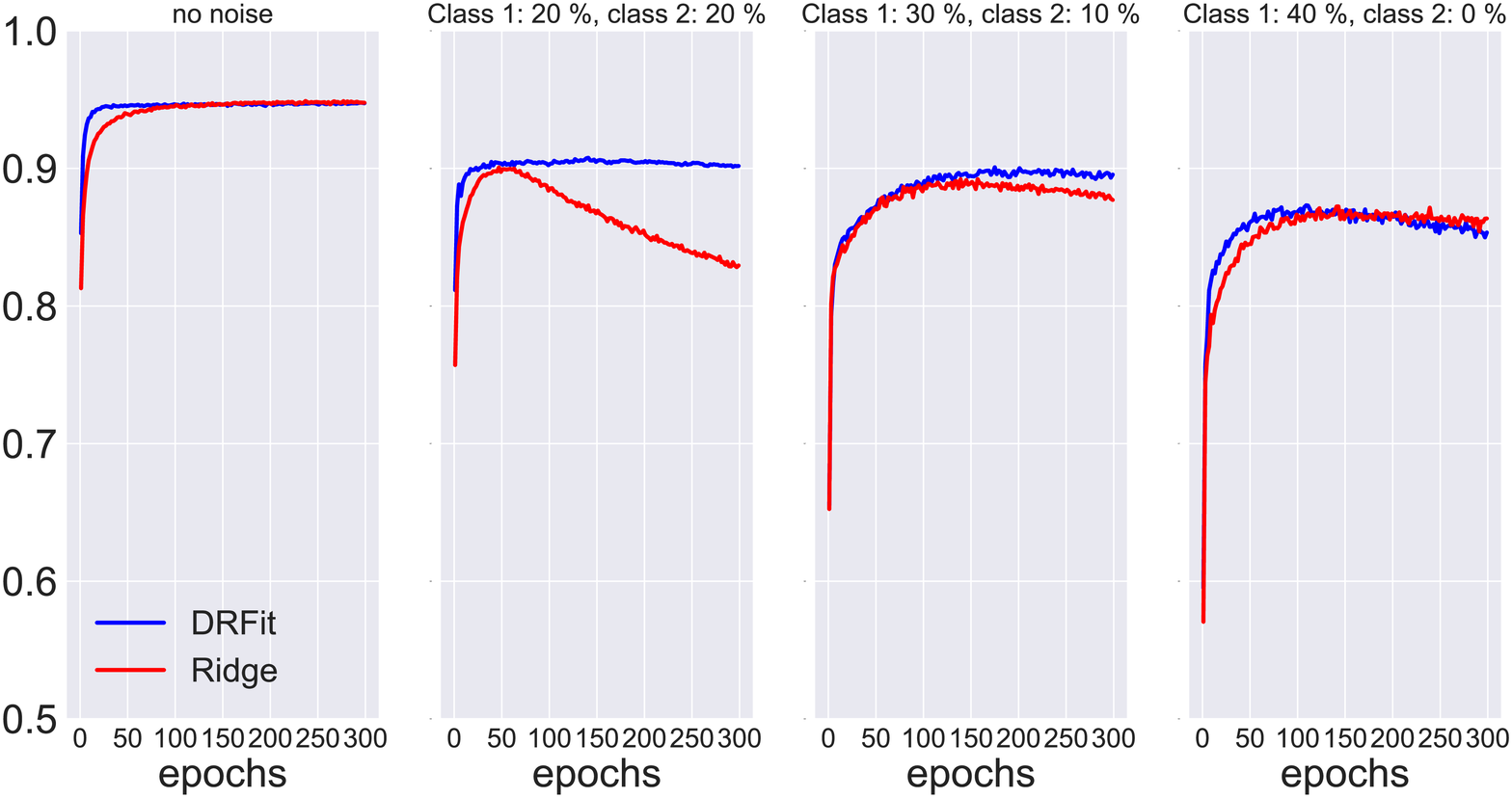}
    \caption{The average accuracy on a test data for DRFit and pure ridge regularisation during training when trained on the classes cars and planes in the CIFAR-10 dataset. In order to vary the noise we have done this in four different label noise settings. One with equally amount of noise in both classes, one with $30 \%$ noise in class one and $10 \%$ in class two, one with $40 \%$ noise in class one and no noise in class two and one setting with no label noise in the training data. In all settings we have assumed that we don't know anything about the noise distribution so $\rho_k = 1$. In order to get a better comparison the runs that crashed are removed.}
    \label{fig:runs_on_test_data_cifar10_appendix}
\end{figure}

\end{document}